\documentclass{llncs}

\usepackage[utf8]{inputenc}
\usepackage[T1]{fontenc}

\usepackage[bookmarks=true,pdfborder={0 0 0},pdfusetitle]{hyperref}

\usepackage{amsmath}
\usepackage{amsfonts}
\usepackage{booktabs}
\usepackage[algo2e, noend]{algorithm2e}
\usepackage{multirow}

\usepackage{microtype}

\usepackage{changebar}
\usepackage{comment}
\includecomment{fullversion}
\excludecomment{shortversion}

\usepackage{todonotes}

\usepackage[smaller]{acronym}

\acrodef{DFPN}{Depth-First Proof Number Search}
\acrodef{ATL}{Alternating-time Temporal Logic}
\acrodef{CTL}{Computation Tree Logic}
\acrodef{LTL}{Linear Temporal Logic}
\acrodef{PDL}{Propositional Dynamic Logic}
\acused{DFPN}
\acused{ATL}
\acused{CTL}
\acused{LTL}
\acused{PDL}
\acrodef{GA}{Game Automaton}
\acrodefplural{GA}{Game Automata} 
\acrodef{MPS}{Minimal Proof Search}
\acrodef{PNS}{Proof Number Search}
\acrodef{MMLK}{Multi-agent Modal Logic K}

\newcommand{\ptimecomplete}{\textsc{ptime}-complete}
\newcommand{\npcomplete}{\textsc{np}-complete}
\newcommand{\gamename}{\textsc}

\DeclareMathOperator*{\argmin}{arg\,min}

\DeclareMathOperator{\MPN}{MPN}
\DeclareMathOperator{\MDN}{MDN}

\DeclareMathOperator{\FORALL}{\Box}
\DeclareMathOperator{\EXISTS}{\Diamond}
\DeclareMathOperator{\eqdef}{\stackrel{\mathrm{def}}{=}}
\newcommand{\Prop}{P}

\usepackage{syntax}

\author{Abdallah Saffidine}
\institute{LAMSADE, Université Paris-Dauphine\\
  \email{abdallah.saffidine@dauphine.fr}}

\title{Minimal Proof Search for Modal Logic K Model Checking\protect{\thanks{A shorter version of this article appears in the proceedings of JELIA 2012.}}}
\begin{shortversion}
\title{Minimal Proof Search for Modal Logic K Model Checking\protect{\thanks{A longer version of this article is available at \url{http://arxiv.org/abs/1207.1832v1}.}}}
\end{shortversion}

\begin{document}

\maketitle
\begin{abstract}
  Most modal logics such as S5, \ac{LTL}, or \ac{ATL} are extensions of Modal Logic K.
  While the model checking problems for \ac{LTL} and to a lesser extent \ac{ATL} have been very active research areas for the past decades,
  the model checking problem for the more basic \ac{MMLK} has important applications as a formal framework for perfect information multi-player games on its own.
  
  We present \ac{MPS}, an effort number based algorithm solving the model checking problem for \ac{MMLK}.
  We prove two important properties for \ac{MPS} beyond its correctness.
  The (dis)proof exhibited by \ac{MPS} is of minimal cost for a general definition of cost,
  and \ac{MPS} is an optimal algorithm for finding (dis)proofs of minimal cost.
  Optimality means that any comparable algorithm either needs to explore a bigger or equal state space than \ac{MPS}, or is not guaranteed to find a (dis)proof of minimal cost on every input.

  As such, our work relates to A* and AO* in heuristic search, to Proof Number Search and \ac{DFPN}+ in two-player games, and to counterexample minimization in software model checking.
\end{abstract}

\acresetall
\acused{DFPN}
\acused{ATL}
\acused{CTL}
\acused{LTL}
\acused{PDL}
\acused{MCPS}
\acused{PNPS}

\section{Introduction}

Model checking for temporal logics such as \ac{LTL} or \ac{CTL} is a major research area with important applications in software and hardware verification~\cite{ClarkeGP1999}.
Model checking for agent logics such as \ac{ATL} or S5 is now also regarded as an important topic with a variety of applications~\cite{vanderHoekW2002,vanDitmarschHK2003,LomuscioQR2009}.
On the other hand, Modal Logic K is usually considered the basis upon which more elaborate modal logics are built, such as S5, \ac{PDL}, \ac{LTL}, \ac{CTL}, or \ac{ATL}~\cite{BlackburnDeRijkeVenema2002,ShohamLeytonBrown2009}. 
\ac{MMLK} can also be used directly to model (sequential) perfect information games.

In this article, we put forward a model checking algorithm for \ac{MMLK} that we call \acf{MPS}.
As the name indicates, given a model checking problem $q \models \phi$, the \ac{MPS} algorithm outputs a proof that $q$ satisfies $\phi$ or a counterexample, this proof/counterexample being minimal for some definition of size.
Perfect information games provide at least two motivations for small proofs.
In game playing, people are usually interested in ``short'' proofs, for instance a \gamename{chess} player would rather deliver checkmate in three moves than in nine moves even if both options grant them the victory.
In game solving, ``compact'' proofs can be stored and independently checked efficiently.

In \ac{CTL} model checking, finding a minimal witness/counterexample is \npcomplete{}~\cite{ClarkeGMZ1995}.
\ac{MMLK} model checking, on the contrary, though \ptimecomplete~\cite{Lange2006JAL}, allows finding minimal witnesses/counterexamples relatively efficiently as we shall see in this article.

Our goal is related both to heuristic search and software model checking.
On one hand, the celebrated A* algorithm outputs a path of minimal cost from a starting state to a goal state.
This path can be seen as the proof that the goal state is reachable, and the cost of the path is the size of the proof.
On the other hand, finding small counterexamples is an important subject in software model checking.
For a failure to meet a specification often indicates a bug in the program, and a small counterexample makes finding and correcting the bug easier~\cite{GroceV2003}.

Like A*, \ac{MPS} is optimal, in the sense that any algorithm provided with the same information and guaranteed to find a proof of minimal size needs to do as many node expansions as \ac{MPS}.

The tableau-based model checking approach by Cleaveland for the $\mu$-calculus seems to be similar to ours~\cite{Cleaveland1989}, however it would need to be adapted to handle (dis)proof cost.
Also, in our understanding, the proof procedure \emph{check1} presented by Cleaveland can be seen as an unguided depth first search while our approach is guided towards regions of minimal cost.

The two algorithms most closely related to \ac{MPS} are AO*, a generalization of A* to And/Or trees, and \ac{DFPN}+~\cite{Nagai2002}, a variant of \ac{DFPN}, itself a depth-first variant of \ac{PNS}~\cite{AllisvdMvdH1994}.

While And/Or trees are as expressive as the combination of \ac{MMLK} and \acp{GA}, we believe that the separation of concerns between the logic and the \acl{GA} is beneficial in practice.
For instance, if the properties to be checked are encoded in the logic rather than in the graph, there is no need to rewrite the rules of \gamename{chess} if one is interested in finding helpmates instead of checkmates, or if one just wants to know if any piece can be captured in two moves from a given position.
The encoding through an And/Or graph would be different in every such situation while in our approach, only the modal logic formula needs to be adapted.
Another advantage of \ac{MPS} over AO* is that if the problem is not solvable, then \ac{MPS} finds a minimal disproof while AO* does not provide such a guarantee.\footnote{Following the convention in \acl{PNS}, we use the term proof and disproof instead of witness and counterexample which are more common in the model checking literature.}

\ac{DFPN}+ is typically only used to find a winning strategy for either player in two-player games.
\ac{MPS}, on the contrary, can be applied to solve other interesting problems without a cumbersome And/Or graph prior conversion.  Example of such problems range from finding ladders in two-player games to finiding paranoid wins in multi-player games.
Another improvement over \ac{DFPN}+ is that we allow for a variety of (dis)proof size definitions.
While \ac{DFPN}+ is set to minimize the total edge cost in the proof, we can imagine minimizing, say, the number of leaves or the depth of the (dis)proof.

In his thesis, Nagai derived the \ac{DFPN} algorithm from the equivalent best-first algorithm \ac{PNS}~\cite{Nagai2002}.
Similarly, we can obtain a depth-first version of \ac{MPS} from the best first search version presented here by adapting Nagai's transformation.
Such a depth-first version should probably be favoured in practice, however we decided to present the best first version in this article for two main reasons.
We believe the best-first search presentation is more accessible to the non-specialists.
The proofs seemed to be easier to work through in the chosen setting, and they can later be extended to the depth-first setting.

The remainder of this paper is structured as follows.  
In Sect.~\ref{sec:definitions} we recall the definitions of \acf{GA} and \ac{MMLK} and formally define (dis)proofs for the corresponding model checking problem.
Section~\ref{sec:mps} elaborates on the notion of (dis)proof cost and the associated basic admissible heuristic functions, it then proceeds with the presentation of the \ac{MPS} algorithm.
Finally, we prove the correctness of \ac{MPS}, the minimality of the output (dis)proofs and the optimality of the algorithm in Sect.~\ref{sec:theory}.
A short discussion concludes the article.

\section{Definitions}
\label{sec:definitions}

We define in this section various formal objects that will be used throughout the paper.
The \ac{GA} is the underlying system which is to be formally verified.
The \ac{MMLK} is the language to express the various properties we want to model check \acp{GA} against.
Finally, a (dis)proof is a tree structure that shows whether a property is true on a state in a \ac{GA}.

\subsection{\aclp{GA}}
A \ac{GA} is a kind of labelled transition system where both the states and the transitions are labelled.
If a \ac{GA} is interpreted as a perfect information game, then a transition corresponds to a move from one state to the next and its label is the player making that move. 
The state labels are domain specific information about states, for instance we could have a label for each triple (piece, owner, position) in \gamename{chess}-like games.
Naturally, it is also possible to give a formal definition of \acp{GA}.

\begin{definition}
  A \emph{\acl{GA}} is a 5-tuple $G = (\Prop, \Sigma, Q, \pi, \delta)$ with the following components:
  \begin{itemize}
  \item $\Prop$ is a non-empty set of \emph{atoms} (or state labels)
  \item $\Sigma$ is a non-empty finite set of \emph{agents} (or transition labels)
  \item $Q$ is a set of \emph{game states}
  \item $\pi : Q \rightarrow 2^\Prop$ maps each state $q$ to its labels
  \item $\delta : Q \times \Sigma \rightarrow 2^Q$ is a transition function that maps a state and an agent to a set of next states.
  \end{itemize}
\end{definition}

In the following, we will use $p$, $p^\prime$, $p_1$, \dots for atoms, $a$ for an agent, and $q$, $q^\prime$, $q_1$, \dots for game states.
We write $q \xrightarrow{a} q^\prime$ when $q^\prime \in \delta(q, a)$ and we read \emph{agent $a$ can move from $q$ to $q^\prime$}.
Note that $\delta$ returns the set of successors, so it need not be a partial function to allow for states without successors.
If an agent $a$ has no moves in a state $q$, we have $\delta(q,a) = \emptyset$.

\subsection{\acl{MMLK}}
Following loosely~\cite{BlackburnDeRijkeVenema2002}, we define the \acl{MMLK} over a set of atoms $\Prop$ as the formulas we obtain by combining the negation and conjunction operators with a set of \emph{box} operators, one per agent.
\begin{definition}
  The set $T$ of well-formed \emph{\acf{MMLK}} formulas is defined inductively as
  $\phi$ $:=$ $p$ \emph{|} $\neg \phi'$ \emph{|} $\FORALL_a \phi'$ \emph{|} $\phi_1 \wedge \phi_2$
  where $\phi$, $\phi'$, $\phi_1$,\dots stand for arbitrary \ac{MMLK} formulas
\end{definition}

We can define the usual syntactic shortcuts for the disjunction and the \emph{diamond} operators $\phi_1 \vee \phi_2 \eqdef \neg(\neg \phi_1 \wedge \neg \phi_2)$ and $\EXISTS_a \phi \eqdef \neg \FORALL_a \neg \phi$.
The box operators convey necessity and the diamond operators convey possibility:
$\FORALL_a \phi$ can be read as \emph{it is necessary for agent $a$ that $\phi$}, while $\EXISTS_a \phi$ is \emph{it is possible for $a$ that $\phi$}.

\subsection{The Model Checking Problem}
We can now interpret \ac{MMLK} formulas over \acp{GA} via the satisfaction relation $\models$.
Intuitively, a state in a \ac{GA} constitutes the context of a formula, while a formula constitutes a property of a state.
A formula might be satisfied in some contexts and not satisfied in other contexts, and some properties hold in a state while others do not.
Determining whether a given formula $\phi$ holds in a given state $q$ (in a given implicit \ac{GA}) is what is commonly referred to as \emph{the model checking problem}.
If it is the case, we write $q \models \phi$, otherwise we write $q \not\models \phi$.

It is possible to decide whether $q \models \phi$ by examining the structure of $\phi$, the labels of $q$, as well as the accessible states.

\begin{definition}
  The formulas \emph{satisfied} by a state $q$ can be constructed by induction as follows.
  \begin{itemize}
  \item If $p$ is a label of $q$, that is if $p \in \pi(q)$, then $q \models p$;
  \item if $q \not\models \phi$ then $q \models \neg \phi$;
  \item if $q \models \phi_1$ and $q \models \phi_2$ then $q \models \phi_1 \wedge \phi_2$;
  \item if for all $q^\prime$ such that $q \xrightarrow{a} q^\prime$, we have $q^\prime \models \phi$, then $q \models \FORALL_a \phi$.
  \end{itemize}
\end{definition}

\subsection{Proofs and Counterexamples}
\label{sec:proofs}

In practice, we never explicitly construct the complete set of formulas satisfied by a state.
So when some computation tells us that a formula $\phi$ is indeed (not) satisfied by a state $q$, some sort of evidence might be desirable.
In software model checking, a model of the program replaces the \ac{GA}, and a formula in a temporal logic acts as a specification of the program.
If a correct model checker asserts that the program does not satisfy the specification, it means that the program or the specification contained a bug.
In those cases, it can be very useful for the programmers to have access to an evidence by the model checker of the mismatch between the formula and the system as it is likely to lead them to the bug.

In this section we give a formal definition of what constitutes a \emph{proof} or a \emph{disproof} for the class of model checking problems we are interested in.
It is possible to relate the following definitions to the more general concept of \emph{tree-like} counterexamples used in model checking ACTL~\cite{ClarkeJLV2002}.
\begin{definition}
  \label{def:tree}
  An \emph{exploration tree} for a formula $\phi$ in a state $q$ is a tree with root $n$ associated with a pair $(q, \phi)$ with $q$ a state and $\phi$ a formula, such that $n$ satisfies the following properties.
  \begin{itemize}
  \item If $n$ is associated with $(q, p)$ with $p \in \Prop$, then it has no children;
  \item if $n$ is associated with $(q, \neg \phi)$ then $n$ has at most one child and it is an exploration tree associated with $(q, \phi)$;
  \item if a node $n$ is associated with $(q, \phi_1\wedge\phi_2)$ then any child of $n$ (if any) is an exploration tree associated with $(q, \phi_1)$ or with $(q, \phi_2)$;
  \item if a node $n$ is associated with $(q, \FORALL_a \phi)$ then any child of $n$ (if any) is an exploration tree associated with $(q^\prime, \phi)$ for some $q^\prime$ such that $q \xrightarrow{a} q^\prime$.
  \item In any case, no two children of $n$ are associated with the same pair.
  \end{itemize}
\end{definition}

Unless stated otherwise, we will not distinguish between a tree and its root node. 
In the rest of the paper, $n$, $n^\prime$, $n_1$, \dots will be used to denote nodes in exploration trees.

\begin{definition}
  \label{def:proof}
  A \emph{proof} (resp. a \emph{disproof}) that $q \models \phi$ is an exploration tree with a root $n$ associated with $(q, \phi)$ satisfying the following hypotheses.
  \begin{itemize}
  \item If $\phi = p$ with $p \in \Prop$, then $p \in \pi(q)$ (resp.~$p \notin \pi(q)$);
  \item if $\phi = \neg \phi^\prime$, then $n$ has exactly one child $n^\prime$ and this child is a disproof (resp.~proof);
  \item if $\phi = \phi_1 \wedge \phi_2$, then $n$ has exactly two children $n_1$ and $n_2$ such that $n_1$ is a proof that $q \models \phi_1$ and $n_2$ is a proof that $q\models \phi_2$ (resp.~$n$ has exactly one child $n^\prime$ and $n^\prime$ is a disproof that $q\models \phi_1$ or $n^\prime$ is a disproof that $q \models \phi_2$);
  \item if $\phi = \FORALL_a \phi^\prime$, then $n$ has exactly one child $n^\prime$ for each $q \xrightarrow{a} q^\prime$, and $n^\prime$ is a proof for $q^\prime \models \phi^\prime$ (resp.~$n$ has exactly one child $n^\prime$ and $n^\prime$ is a disproof for $q^\prime \models \phi^\prime$ for some $q \xrightarrow{a} q^\prime$).
  \end{itemize}
\end{definition}



\section{\acl{MPS}}
\label{sec:mps}
Let $q \models \phi$ be a model checking problem and $n_1$ and $n_2$ two proofs as defined in Sect.~\ref{sec:proofs}.  
Even if $n_1$ is not a subtree of $n_2$, there might be reasons to prefer, $n_1$ over $n_2$.  
For instance, we can imagine that $n_1$ contains fewer nodes than $n_2$, or that the depth of $n_1$ is smaller than that of $n_2$.

\subsection{Cost Functions}
\label{sec:cost}
To remain as general as possible with respect to the definitions of a \emph{small} (dis)proof in the introduction, we introduce a cost function $k$ as well as cost aggregators $A_\wedge$ and $A_{\FORALL}$.  
These functions can then be instantiated in a domain dependent manner to get the optimal algorithm for the domain definition of minimality.
This approach has been used before in the context of A* and AO*~\cite{Pearl1984}.

We assume given a \emph{base cost function} $k: \Prop \rightarrow \mathbb{R}^+$, as well as a \emph{conjunction cost aggregator} $A_\wedge : \mathbb{N}^{\mathbb{R}^+\cup\{\infty\}} \rightarrow \mathbb{R}^+\cup\{\infty\}$ and a \emph{box modal cost aggregator} $A_{\FORALL} : \Sigma \times \mathbb{N}^{\mathbb{R}^+\cup\{\infty\}} \rightarrow \mathbb{R}^+\cup\{\infty\}$,
where $\mathbb{N}^{\mathbb{R}^+\cup\{\infty\}}$ denotes the set of multisets of $\mathbb{R}^+\cup\{\infty\}$.

We assume the aggregators are increasing in the sense that adding elements to the input increases the cost.
For all costs $x\leq y\in \mathbb{R}^+\cup\{\infty\}$, multisets of costs $X \in \mathbb{N}^{\mathbb{R}^+\cup\{\infty\}}$, and for all agents $a$, we have for the conjunction cost aggregator $A_\wedge(X) \leq A_\wedge(\{x\} \cup X) \leq A_\wedge(\{y\}\cup X)$, and for the box aggregator $A_{\FORALL}(a,X) \leq A_{\FORALL}(a,\{x\} \cup X) \leq A_{\FORALL}(a,\{y\}\cup X)$.

We further assume that aggregating infinite costs results in infinite costs and that aggregating finite numbers of finite costs results in finite costs. 
For all costs $x\in\mathbb{R}^+$, multisets of costs $X\in\mathbb{N}^{\mathbb{R}^+\cup\{\infty\}}$, and for all agents $a$, $A_{\wedge}(\{\infty\}) = A_{\FORALL}(a,\{\infty\}) = \infty$ and that $A_{\wedge}(X) < \infty \Rightarrow A_{\wedge}(\{x\}\cup X) < \infty$ and $A_{\FORALL}(a, X) < \infty \Rightarrow A_{\FORALL}(a, \{x\}\cup X) < \infty$.

Note that in our presentation, there is no cost to a negation.  
The justification is that we want a proof aggregating over a disjunction to cost as much as a disproof aggregating over a conjunction with children of the same cost, without having to include the disjunction and the diamond operator in the base syntax.

Given $k$, $A_\wedge$, and $A_{\FORALL}$, it is possible to define the \emph{global cost function} for a (dis)proof as shown in Table~\ref{tab:cost}.
\begin{table}
  \centering
  \caption{Cost $K$ of a proof or a disproof for a node $n$ as a function of the base cost function $k$ and the aggregators $A_\wedge$ and $A_{\FORALL}$.
    $C$ is the set of children of $n$.}
  \label{tab:cost}
  \begin{tabular}{lrr}
    \toprule
    Label of $n$ & Children of $n$ & $K(n)$ \\
    \midrule
    $(q, p)$ & $\emptyset$ & $k(p)$ \\
    $(q, \neg \phi)$ & $\{c\}$ & $K(c)$ \\
    $(q, \phi_1 \wedge \phi_2)$ & $C$ & $A_\wedge(\{K(c)| c \in C\})$ \\
    $(q, \FORALL_a \phi)$ & $C$ & $A_{\FORALL}(a, \{K(c)| c \in C\})$ \\
    \bottomrule
  \end{tabular}
\end{table}

\begin{example}
  Suppose we are interested in the nested depth of the $\FORALL$ operators in the (dis)proof.
  Then we define $k = 0$, $A_\wedge = \max$, and $A_{\FORALL}(a, X) = 1 + \max X$ for all $a$.
\end{example}
\begin{example}
  Suppose we are interested in the number of atomic queries to the underlying system (the \ac{GA}).
  Then we define $k = 1$, $A_\wedge(X) = \sum X$, and $A_{\FORALL}(a, X) = \sum X$ for all $a$.
\end{example}
\begin{fullversion}
\begin{example}
  \label{ex:real}
  Suppose we are interested in minimizing the amount of expansive interactions with the underlying system.
  Then we define $A_\wedge(X) = \sum X$, and $A_{\FORALL}(a, X) = k_{\FORALL_a} + \sum X$ for all $a$.  
  In this case, we understand that $k(p)$ is the price for querying $p$ in any state, and $k_{\FORALL_a}$ is the price for getting access to the transition function for agent $a$ in any state. 
\end{example}
\end{fullversion}

We define two \emph{heuristic} functions $I$ and $J$ to estimate the minimal amount of interaction needed with the underlying system to say anything about a formula $\phi$.
These functions are defined in Table~\ref{tab:heuristic}, $I(\phi)$ is a lower bound on the minimal amount of interaction to prove $\phi$ and $J(\phi)$ is a lower bound on the minimal amount of interaction to disprove $\phi$.

\begin{table}
  \centering
  \caption{Definition of the heuristic functions $I$ and $J$.}
  \label{tab:heuristic}
  \begin{tabular}{lrrr}
    \toprule
    Shape of $\phi$ & $I(\phi)$ & & $J(\phi)$ \\
    \midrule
    $p$ & $k(p)$ & & $k(p)$ \\
    $\neg \phi^\prime$ & $J(\phi^\prime)$ & & $I(\phi^\prime)$ \\
    $\phi_1 \wedge \phi_2$ & $A_{\wedge}(\{I(\phi_1), I(\phi_2)\})$ & & $\min_{i\in\{1,2\}}A_{\wedge}(\{J(\phi_i)\})$ \\
    $\FORALL_a \phi^\prime$ & $A_{\FORALL}(a, \emptyset)$ & & $A_{\FORALL}(a,\{J(\phi^\prime)\})$ \\
    \bottomrule
  \end{tabular}
\end{table}

The heuristics $I$ and $J$ are \emph{admissible}, that is, they never overestimate the cost of a (dis)proof.

\begin{proposition}
\label{prop:admissibility}
Given a formula $\phi$, for any state $q$, for any proof $n$ that $q \models \phi$ (resp.~disproof), $I(\phi) \leq K(n)$ (resp.~$J(\phi) \leq K(n)$).
\end{proposition}
\begin{fullversion}
\begin{proof}
  We proceed by structural induction on the shape of formulas.
  For the base case $\phi = p$, if $n$ is a proof that $q \models p$, then the $n$ label of $n$ is $(q, p)$ and its cost is $K(n) = k(p)$, which is indeed greater or equal to $I(p) = J(p) = k(p)$.

  For the induction case, take the formulas $\phi_1$ and $\phi_2$ and assume that for any proofs (resp.~disproofs) $n_1$ and $n_2$, the cost is greater than the heuristic value: $I(\phi_1) \leq K(n_1)$ and $I(\phi_2) \leq K(n_2)$ (resp.~$J(\phi_1) \leq K(n_1)$ and $J(\phi_2) \leq K(n_2)$).
  
  For any proof (resp.~disproof) $n$ with label $(q, \neg \phi_a)$ and child $c$, the cost of $n$ is the cost of the disproof (resp.~proof) $c$: $K(n) = K(c)$.
  The disproof (resp.~proof) $c$ is associated with $(q, \phi_1)$ and we know from the induction hypothesis that $J(\phi_1) \leq K(c)$ (resp.~$I(\phi_1) \leq K(c)$).
  By definition of the heuristics, $I(\phi) = J(\phi_1)$ (resp.~$J(\phi) = I(\phi_1)$), therefore we have $I(\phi) \leq K(n)$ (resp.~$J(\phi) \leq K(n)$).

  For any proof (resp.~disproof) $n$ with label $(q, \phi_1 \wedge \phi_2)$ and children $c_1, c_2$ (resp.~child $c$), the cost of $n$ is the sum of the costs of the children: $K(n) = K(c_1)+K(c_2)$ (resp.~$K(n)=K(c)$).
  The nodes $c_1$ and $c_2$ are associated with $(q, \phi_1)$ and $(q, \phi_2)$ (resp.~$c$ is associated with $(q, \phi_1)$ or to $(q, \phi_2)$) and we know from the induction hypothesis that $I(\phi_1) \leq K(c_1)$ and $I(\phi_2) \leq K(c_2)$ (resp.~$J(\phi_1) \leq K(c)$ or $J(\phi_2) \leq K(c)$).
  By definition of the heuristics, $I(\phi) = I(\phi_1) + I(\phi_2)$ (resp.~$J(\phi) = \min\{J(\phi_1), J(\phi_2)\}$), therefore we have $I(\phi) \leq K(n)$ (resp.~$J(\phi) \leq K(n)$).

  The remaining case is very similar and is omitted.
  \qed
\end{proof}
\end{fullversion}
\begin{fullversion}
\begin{lemma}
  \label{lemma:heuristics}
  For any formula $\phi$, $I(\phi) < \infty$ and $J(\phi) < \infty$.
\end{lemma}
\begin{proof}
  We proceed by structural induction on $\phi$.
  For the base case, $\phi = p$, simply recall that the range of $k$ is $\mathbb{R}^+$.
  The induction case results directly from the assumptions on the aggregators.
\end{proof}
\end{fullversion}
\subsection{Best First Search framework}
We inscribe the \ac{MPS} algorithm in a best first search framework inspired by game tree search.
We then specify a function for initializing the leaves, a function to update tree after a leaf has been expanded, a selection function to decide which part of the tree to expand next, and a stopping condition for the overall algorithm.

\begin{algorithm2e}
  \SetKw{Skip}{skip}
  \SetKwFunction{Bfs}{bfs}
  \SetKwFunction{BackProp}{backpropagate}
  \SetKwFunction{Update}{update}
  \SetKwFunction{Extend}{extend}
  \SetKwFunction{Select}{select-child}
  \SetKwFunction{Init}{init-leaf}
  \SetKwFunction{InitTerminal}{info-term}

  \Bfs{state $q$, formula $\phi$}\\
  $r$ $\leftarrow$ new node with label $(q, \phi)$\;
  $r$.info $\leftarrow$ \Init{$r$}\;
  $n$ $\leftarrow$ $r$\;
  \While{$r$ is not solved}{
    \While{$n$ is not a leaf}{
      $n$ $\leftarrow$ \Select{$n$}\;
    }
    \Extend{$n$}\;
    $n$ $\leftarrow$ \BackProp{$n$}\;
  }
  \Return{$r$}

  \BlankLine
  \BlankLine
  \Extend{node $n$}\\
  \Switch{on the label of $n$}{
    \Case{$(q, p)$}{
      $n$.info $\leftarrow$ \InitTerminal{$n$}\;
    }
    \Case{$(q, \phi_1 \wedge \phi_2)$}{
      $n_1$ $\leftarrow$ new node with label $(q, \phi_1)$\;
      $n_2$ $\leftarrow$ new node with label $(q, \phi_2)$\;
      $n_1$.info $\leftarrow$ \Init{$n_1$}\;
      $n_2$.info $\leftarrow$ \Init{$n_2$}\;
      Add $n_1$ and $n_2$ as children of $n$\;
    }
    \Case{$(q, \neg\phi_1)$}{
      $n^\prime$ $\leftarrow$ new node with label $(q, \phi_1)$\;
      $n^\prime$.info $\leftarrow$ \Init{$n^\prime$}\;
      Add $n^\prime$ as a child of $n$\;
    }
    \Case{$(q, \FORALL_a\phi_1)$}{
      \ForEach{$q^{\prime}$ in $\{q^{\prime}, q\xrightarrow{a}q^{\prime}\}$}{
        $n^\prime$ $\leftarrow$ new node with label $(q^\prime, \phi_1)$\;
        $n^\prime$.info $\leftarrow$ \Init{$n^\prime$}\;
        Add $n^\prime$ as child of $n$\;
      }
    }
  }

  \BlankLine
  \BlankLine
  \BackProp{node $n$}\\
  new\_info $\leftarrow$ \Update{$n$}\;
  \lIf{new\_info = $n$.info $\vee$ $n$ = $r$}{
    \Return{$n$}\;
  }\Else{
    $n$.info $\leftarrow$ new\_info\;
    \Return{\BackProp{$n$.parent}}
  }

  \BlankLine
  \caption[Best-first search algorithm]{Generic pseudo-code for a best-first search algorithm.}
  \label{alg:bfs}
\end{algorithm2e}

Algorithm~\ref{alg:bfs} develops an exploration tree for a given state $q$ and formula $\phi$.  
To be able to orient the search efficiently towards proving or disproving the model checking problem $q \models \phi$ instead of just exploring, we need to attach additional information to the nodes beyond their $($state, formula$)$ label.
This information takes the form of two \emph{effort} numbers, called the \emph{minimal proof number} and \emph{minimal disproof number}.
Given a node $n$ associated with a pair $(q, \phi)$, the minimal proof number of $n$, $\MPN(n)$, is an indication on the cost of a proof for $q \models \phi$.  
Conversely, the minimal disproof number of $n$, $\MDN(n)$, is an indication on the cost of a disproof for $q \models \phi$.
For a more precise relationship between $\MPN(n)$ and the cost of a proof see Prop.~\ref{prop:lower-bound}.

The algorithm stops when the minimal (dis)proof number reaches $\infty$ as it corresponds to the exploration tree containing a (dis)proof of minimal cost (see Prop.~\ref{prop:solving}).

\begin{table}
  \centering
  \caption{Values for terminal nodes and initial values for leaves.}
  \label{tab:init}
  \begin{tabular}{llrr}
    \toprule
    & Node label & $\MPN$ & $\MDN$ \\
    \midrule
    \multirow{2}{*}{\texttt{info-term}}    & $(q, p)$ where $p \in \pi(q)$  & $k(p)$ & $\infty$\\
    & $(q, p)$ where $p \notin \pi(q)$  & $\infty$ & $k(p)$ \\
    \cmidrule{1-4}
    \texttt{init-leaf} & $(q, \phi)$ & $I(\phi)$ & $J(\phi)$\\
    \bottomrule
  \end{tabular}
\end{table}

The values for the effort numbers in terminal leaves and in newly created leaves are defined in Table~\ref{tab:init}.
The values for the effort numbers of an internal node as a function of its children are defined in Table~\ref{tab:heredity}.
Finally, the selection procedure base on the effort numbers to decide how to descend the global tree is given in Table~\ref{tab:selection}. 
The stopping condition, Table~\ref{tab:init},~\ref{tab:heredity}, and~\ref{tab:selection}, as well as Alg.~\ref{alg:bfs} together define \acl{MPS}.

\begin{table}
  \centering
  \caption{Determination of values for internal nodes.}
  \label{tab:heredity}
  \begin{tabular}{lrrr}
    \toprule
    Node label & Children  &  $\MPN$ & $\MDN$ \\
    \midrule
    $(q, \neg \phi)$      & $\{c\}$ & $\MDN(c)$ & $\MPN(c)$ \\
    $(q, \phi_1 \wedge \phi_2)$    &     $C$ & $A_{\wedge}(\{\MPN(c)|c\in C\})$ & $\min_CA_{\wedge}(\{\MDN\})$ \\
    $(q, \FORALL_a\phi)$ &     $C$ & $A_{\FORALL}(a,\{\MPN(c)|c\in C\})$ & $\min_CA_{\FORALL}(a,\{\MDN\})$ \\
    \bottomrule
  \end{tabular}
\end{table}

\begin{table}[tbp]
  \centering
  \caption{Selection policy.}
  \label{tab:selection}
  \begin{tabular}{lrr}
    \toprule
    Node label & Children  &  Chosen child \\
    \midrule
    $(q, \neg \phi)$      & $\{c\}$ & $c$ \\
    $(q, \phi_1 \wedge \phi_2)$    &     $C$ & $\argmin_CA_{\wedge}(\{\MDN\})$  \\
    $(q, \FORALL_a \phi)$ &     $C$ & $\argmin_CA_{\FORALL}(a,\{\MDN\})$ \\
    \bottomrule
  \end{tabular}
\end{table}

The \texttt{backpropagate} procedure implements a small optimization known as the \emph{current node enhancement}~\cite{AllisvdMvdH1994}.
Basically, if the information about a node $n$ are not changed, then the information about the ancestors of $n$ will not change either and so the next descend will reach $n$.
Thus, it is possible to shortcut the process and start the next descent at $n$ directly.

\section{Properties of \acs{MPS}}
\label{sec:theory}
Before studying some theoretical properties of (dis)proofs, minimal (dis)proof numbers, and \ac{MPS},
let us point out that for any exploration tree, not necessarily produced by \ac{MPS}, we can associate to each node an $\MPN$ and an $\MDN$ by using the initialization described in Table~\ref{tab:init} and the heredity rule described in Table~\ref{tab:heredity}.

\subsection{Correctness of the Algorithm}
The first property we want to prove about \ac{MPS} is that the descent does not get stuck in a solved subtree.
\begin{proposition}
  For any internal node $n$ with finite effort numbers, the child $c$ selected by the procedure described in Table~\ref{tab:selection} has finite effort numbers.
 $\MPN(n) \neq \infty$ and $\MDN(n) \neq \infty$ imply $\MPN(c) \neq \infty$ and $\MDN(c) \neq \infty$.
\end{proposition}
\begin{fullversion}
\begin{proof}
  If the formula associated with $n$ has shape $\neg \phi$, then $\MDN(c) = \MPN(n) \neq \infty$ and $\MPN(c) = \MDN(n) \neq \infty$.
  
  If the formula associated with $n$ is a conjunction, then it suffices to note that no child of $n$ has an infinite minimal proof number and at least one child has a finite minimal disproof number, and the result follows from the definition of the selection procedure.  
  This also holds if the formula associated with $n$ is of the form $\FORALL_a \phi^\prime$.
  \qed
\end{proof}
\end{fullversion}
As a result, each descent ends in a non solved leaf.
Either the associated formula is of the form $p$ and the leaf gets solved, or the leaf becomes an internal node and its children are associated with structurally smaller formulas.

\begin{proposition}
  The \ac{MPS} algorithm terminates in a finite number of steps.
\end{proposition}
\begin{fullversion}
  \begin{proof}
    Let $F$ be the set of lists of formulas ordered by decreasing structural complexity, that is, $F = \{l = (\phi_0, \dots, \phi_n)| n \in \mathbb{N}, \phi_0 \geq \dots \geq \phi_n\}$.
    Note that the lexicographical ordering (based on structural complexity) $<_F$ is wellfounded on $F$.
    Recall that there is no infinite descending chains with respect to a well-founded relation.

    Consider at some time $t$ the list $l_t$ of formulas associated with the non solved leaves of the tree. 
    Assuming that $l_t$ is ordered by decreasing structural complexity, we have $l_t \in F$.
    Observe that a step of the algorithm results in a list $l_{t+1}$ smaller than $l_t$ according to the lexicographical ordering and that successive steps of the algorithm result in a descending chain in $F$.
    Conclude that the algorithm terminates after a finite number of steps for any input formula $\phi$ with associated list $l_0 = (\phi)$.
    \qed
  \end{proof}
\end{fullversion}

Since the algorithm terminates, we know that the root of the tree will eventually be labelled with a infinite minimal (dis)proof number and thus will be \emph{solved}.
It remains to be shown that this definition of a solved tree coincides with containing (dis)proof starting at the root.

\begin{proposition}
  \label{prop:solving} 
  If a node $n$ is associated with $(q, \phi)$, then $\MDN(n) = \infty$ (resp. $\MPN(n) = \infty$) if and only if the tree corresponding to $n$ contains a proof (resp.~disproof) that $q \models \phi$ as a subtree with root $n$.
\end{proposition}
\begin{fullversion}
  \begin{proof}
    We proceed by structural induction on the shape of trees.

    For the base case when $n$ has no children, either $\phi = p$ or $\phi$ is not atomic.
    In the first case, $n$ is a terminal node so contains a (dis)proof ($n$ itself) and we obtain the result by definition of $\MPN$ nad $\MDN$ as per Table~\ref{tab:init}.
    In the second case, $\phi$ is not atomic and $n$ has no children so $n$ does not contain a proof nor a disproof.
    Table~\ref{tab:init} and Lemma~\ref{lemma:heuristics} show that the effort numbers are both finite.

    For the induction case when $\phi = \neg \phi^\prime$, we know that $n$ has one child $c$ associated to $\phi^\prime$.
    If $c$ contains a proof (resp.~disproof) that $q \models \phi^\prime$, then $n$ contains a disproof (resp.~proof) that $q \models \phi$. 
    By induction hypothesis, we know that $\MPN(c) = \infty$ (resp.~$\MDN(c) = \infty$) therefore, using Table~\ref{tab:heredity}, we know that $\MDN(n) = \infty$ (resp.~$\MPN(n) = \infty$).
    Conversely if $c$ does not contain a proof nor a disproof, then $n$ does not contain a proof nor a disproof, and we know from the induction hypothesis and Table~\ref{tab:heredity} than $\MPN(n)=\MDN(c) < \infty$ and $\MDN(n)=\MPN(c)<\infty$.

    The other induction cases are similar but make use of the assumption that aggregating inifinite costs results in infinite costs and that aggregating finite numbers of finite costs results in finite costs.
    \qed
  \end{proof}
\end{fullversion}
\begin{theorem}
  \label{thm:correctness}
  The \ac{MPS} algorithm takes a formula $\phi$ and a state $q$ as arguments and returns after a finite number of steps an exploration tree that contains a (dis)proof that $q \models \phi$.
\end{theorem}

\subsection{Minimality of the (Dis)Proofs}
Now that we know that \ac{MPS} terminates and returns a tree containing a (dis)proof, we need to prove that this (dis)proof is of minimal cost.

The two following propositions can be proved by a simple structural induction on the exploration tree, using Table~\ref{tab:init} and the admissibility of $I$ and $J$ for the base case and Table~\ref{tab:heredity} for the inductive case.
\begin{proposition}
  \label{prop:cost}
  If a node $n$ is solved, then the cost of the contained (dis)proof is given by the minimal (dis)proof number of $n$.
\end{proposition}
\begin{fullversion}
\begin{proof}
  Straightforward structural induction on the shape of the tree using the first half of Table~\ref{tab:init} for the base case and Table~\ref{tab:heredity} for the induction step.
\qed
\end{proof}
\end{fullversion}
\begin{proposition}
  \label{prop:lower-bound}
  If a node $n$ is associated with $(q, \phi)$, then for any proof $m$ (resp.~disproof) that $q \models \phi$, we have $\MPN(n) \leq K(m)$ (resp.~$\MDN(n)\leq K(m)$).
\end{proposition}
\begin{fullversion}
\begin{proof}
  Structural induction on the shape of the tree, using the second half of Table~\ref{tab:init} and the admissibility of $I$ and $J$ (Prop.~\ref{prop:admissibility}) for the base case and Table~\ref{tab:heredity} for the inductive case.
  \qed
\end{proof}
\end{fullversion}

Since the aggregators for the cost function are increasing functions, then $\MPN(n)$ and $\MDN(n)$ are non decreasing as we add more nodes to the tree $n$.

\begin{proposition}
  \label{prop:minimal-child}
  For each disproved internal node $n$ in a tree returned by the \ac{MPS} algorithm, at least one of the children of $n$ minimizing the $\MDN$ is disproved.
\end{proposition}
\begin{fullversion}
\begin{proof}[Sketch]
  If we only increase the minimal (dis)proof number of a leaf, then for each ancestor, at least one of either the minimal proof number of the minimal disproof number remains constant.

  Take a disproved internal node $n$, and assume we used the selection procedure described in Table~\ref{tab:selection}.
  On the iteration that lead to $n$ being solved, the child $c$ of $n$ selected was minimizing the $\MDN$ and this number remained constant since $\MPN(c)$ raised from a finite value to $\infty$.

  Since the $\MDN$ of the siblings of $c$ have not changed, then $c$ is still minimizing the $\MDN$ after it is solved.
  \qed
\end{proof}
\end{fullversion}

Combining Prop.~\ref{prop:cost},~\ref{prop:lower-bound}, and~\ref{prop:minimal-child}, we get the following theorem.
\begin{theorem}
  The tree returned by the \ac{MPS} algorithm contains a (dis)proof of minimal cost.
\end{theorem}

\subsection{Optimality of the Algorithm}

The \ac{MPS} algorithm is not optimal in the most general sense because it is possible to have better algorithm in some cases by using transpositions, domain knowledge, or logical reasoning on the formula to be satisfied.

For instance, take $\phi_1 = \EXISTS_a (p \wedge  \neg p)$ and $\phi_2$ some non trivial formula satisfied in a state $q$.  
If we run the \ac{MPS} algorithm to prove that $q \models \phi_1 \vee \phi_2$, it will explore at least a little the possibility of proving $q \models \phi_1$ before finding the minimal proof through $\phi_2$.
We can imagine that a more ``clever'' algorithm would recognize that $\phi_1$ is never satisfiable and would directly find the minimal proof through $\phi_2$.

Another possibility to outperform \ac{MPS} is to make use of transpositions to shortcut some computations.
\ac{MPS} indeed explores structures according to the \ac{MMLK} formula shape, and it is well-known in modal logic that bisimilar structures cannot be distinguished by \ac{MMLK} formulas.
It is possible to express an algorithm similar to \ac{MPS} that would take transpositions into account, adapting ideas from \ac{PNS}~\cite{SchijfAU1994,Mueller2003CG,KishimotoMueller2005}.
We chose not to do so in this article for simplicity reasons.

Still, \ac{MPS} can be considered optimal among the programs that do not use reasoning on the formula itself, transpositions or domain knowledge.
Stating and proving this property formally is not conceptually hard, but we have not been able to find simple definitions and a short proof that would not submerge the reader with technicalities.
Therefore we decided only to describe the main ideas of the argument from a high-level perspective.


\begin{definition}
  A pair $(q^\prime, \phi^\prime)$ is \emph{similar} to a pair $(q, \phi)$ with respect to an exploration tree $n$ associated with $(q, \phi)$ if $q^\prime$ can substitute for $q$ and $\phi^\prime$ for $\phi$ in $n$.
\end{definition}


Let $n$ associated with $(q, \phi)$ be an exploration tree with a finite $\MPN$ (resp.~$\MDN$), then we can construct a pair $(q^\prime, \phi^\prime)$ similar to $(q, \phi)$ with respect to $n$ such that there is a proof that $q^\prime \models \phi^\prime$ of cost exactly $\MPN(n)$ (resp.~a disproof of cost $\MDN(n)$).

\begin{definition}
  An algorithm $A$ is \emph{purely exploratory} if the following holds.
  Call $n$ the tree returned by $A$ when run on a pair $(q, \phi)$.
  For every pair $(q^\prime, \phi^\prime)$ similar to $(q, \phi)$ with respect to $n$, running $A$ on $(q^\prime, \phi^\prime)$ returns a tree structurally equivalent to $n$.
\end{definition}

Depth first search, if we were to return the explored tree, and \ac{MPS} are both examples of purely exploratory algorithms.

\begin{proposition}
  If a purely exploratory algorithm $A$ is run on a problem $(q, \phi)$ and returns a solved exploration tree $n$ where $\MPN(n)$ (resp.~$\MDN(n)$) is smaller than the cost of the contained proof (resp.~disproof), then we can construct a problem $(q^\prime, \phi^\prime)$ similar with respect to $n$ such that $A$ will return a structurally equivalent tree with the same proof (resp.~disproof) while there exists a proof of cost $\MPN(n)$ (resp.~disproof of cost $\MDN(n)$).
\end{proposition}

Note that if the cost of a solved exploration tree $n$ is equal to its $\MPN$ (resp.~$\MDN$), then we can make \ac{MPS} construct a solved shared root subtree of $n$ just by influencing the tie-breaking in the selection policy described in Table~\ref{tab:selection}.

\begin{theorem}
  If a purely exploratory algorithm $A$ returns a solved exploration tree $n$, either this tree (or a subtree) can be generated by \ac{MPS} or $A$ is not guaranteed to return a tree containing a (dis)proof of minimal cost on all possible inputs. 
\end{theorem}

\section{Conclusion and discussion}

We presented \acf{MPS}, a model checking algorithm for \ac{MMLK}.
\ac{MPS} has been proven correct, and it has been proved that the (dis)proof returned by \ac{MPS} was minimizing a generic cost function.
The only assumption about the cost function is that it is defined recursively using increasing aggregators.
Finally, we have shown that \ac{MPS} was optimal among the purely exploratory model checking algorithms for \ac{MMLK}.

Nevertheless, the proposed approach has a few limitations.
\ac{MPS} is a best first search algorithm and is memory intensive;
the cost functions addressed in the article cannot represent variable edge cost;
and \ac{MPS} cannot make use of transpositions in its present form.
Still, we think that these limitations can be overcome in future work.

We envision a depth-first adaptation of \ac{MPS} similar to Nagai's transformation of \ac{PNS} into \ac{DFPN}.
Alternatively, we can draw inspiration from PN$^2$~\cite{AllisvdMvdH1994} and replace the heuristic functions $I$ and $J$ by a nested call to \ac{MPS}, leading to an \ac{MPS}$^2$ algorithm trading time for memory.
These two alternative algorithms would directly inherit the correctness and minimality theorems for \ac{MPS}.
The optimality theorem would also transpose in the depth-first case, but it would not be completely satisfactory.
Indeed, even though the explored tree will still be admissibly minimal, several nodes inside the tree will have been forgotten and re-expanded multiple times.
This trade-off is reminiscent of the one between A* and its depth-first variation IDA*~\cite{Korf1985}.

Representing problems with unit edge costs is already possible within the framework presented in Sect.~\ref{sec:cost}.
It is not hard to adapt \ac{MPS} to the more general case as we just need to replace the agent labels on the transitions with (agent, cost) labels. 
This more general perspective was not developed in this article because the notation would be heavier while it would not add much to the intuition and the general understanding of the ideas behind \ac{MPS}.

Finding minimal (dis)proofs while taking transpositions into account is more challenging because of the double count problem.
While it is possible to obtain a correct algorithm returning minimal (dis)proofs by using functions based on propagating sets of individual costs instead of real values in Sect.~\ref{sec:cost}, similarly to previous work in \ac{PNS}~\cite{Mueller2003CG},
such a solution would hardly be efficient in practice and would not necessarily be optimal.
The existing literature on \ac{PNS} and transpositions can certainly be helpful in addressing efficient handling of transpositions in \ac{MMLK} model checking~\cite{SchijfAU1994,Mueller2003CG,KishimotoMueller2005}.

Beside evaluating and improving the practical performance of \ac{MPS}, future work can also study to which extent the ideas presented in this article can be applied to model checking problems in more elaborate modal logics and remain tractable.

\subsubsection{\ackname}
We would like to thank the reviewers for their helpful comments and for pointing out Cleaveland's related work~\cite{Cleaveland1989}.

\bibliographystyle{splncs03}
\bibliography{minimalProofSearch}
\end{document}